\documentclass{article}

\usepackage{graphicx}
\usepackage{latexsym}
\usepackage{fullpage}
\usepackage{graphicx}

\newtheorem{theorem}{Theorem}

\newtheorem{definition}{Definition}
\newenvironment{proof}
{\noindent {\textbf{Proof:}~ }}
{\hfill\rule{2mm}{2mm} \vspace{\parskip} }

\begin{document}

\title {\bf The Complexity of Manipulating $k$-Approval Elections \thanks{Supported in part by NSF grant IIS-0713061. This work has also appeared in the $3^{rd}$ International Conference on Agents and Artificial Intelligence \cite{Lin11}}
}
\author{
Andrew Lin\\
\small Department of Computer Science\\[-0.8ex]
\small Rochester Institute of Technology\\[-0.8ex]
\small Rochester, NY  14623, USA\\[-0.8ex]
\small \texttt{apl8378@cs.rit.edu}\\
}

\date{}

\maketitle

\begin{abstract}

An important problem in computational social choice theory is the complexity of undesirable behavior among agents, such as control, manipulation, and bribery in  election systems. These kinds of voting strategies are often tempting at the individual level but disastrous for the agents as a whole. Creating election systems where the determination of such strategies is difficult is thus an important goal.

An interesting set of elections is that of scoring protocols. Previous work in this area has demonstrated the complexity of misuse in cases involving a fixed number of candidates, and of specific election systems on unbounded number of candidates such as Borda. In contrast, we take the first step in generalizing the results of computational complexity of election misuse to cases of infinitely many scoring protocols on an unbounded number of candidates. Interesting families of systems include $k$-approval and $k$-veto elections, in which voters distinguish $k$ candidates from the candidate set.

Our main result is to partition the problems of these families based on their complexity. We do so by showing they are polynomial-time computable, NP-hard, or polynomial-time equivalent to another problem of interest. We also demonstrate a surprising connection between manipulation in election systems and some graph theory problems.
\end{abstract}

%%%%%%%%%%%%%%%%%%%%%%%%%%%%%%%%%%%%%%%%%%%%%%%%%%%%%%%%%%%%%%%%%%%%%%%%%

\section{\uppercase{Introduction}}

Elections are a means for choosing one or more candidates given the preferences of individuals to arrive at a decision that attempts to maximize the collective welfare of the individuals (see also~\cite{Brams}). A voting system contains both rules for valid voting, and how one yields a final outcome given the preferences of individual voters. The choice of which rule one uses to choose the winner(s) may affect both the outcome of the election and the behavior of each voter. Voters may be unweighted or weighted, and need not count equally toward the final result.

The computational theory behind voting, as well as systems beyond that of the obvious majority and plurality, were first studied during the French Revolution. Two notable early theorists in this field were Jean-Charles de Borda~\cite{Borda} and Marquis de Condorcet~\cite{Condorcet}. Borda introduced the system used for electing members to the French Academy of Sciences in 1770. The Borda system favors candidates that are somewhat liked by a lot of voters. Condorcet believed the winner should be one who fares well in pairwise comparisons of the candidates, a concept that is poorly captured by the Borda election, as it does not view pairs of candidates.

Condorcet discovered that, among voters with transitive preferences (meaning that if a voter prefers $x$ to $y$ and $y$ to $z$, he or she will prefer $x$ to $z$), majority preferences between pairs of candidates may be intransitive. This means that among candidates $x$, $y$, and $z$, it is possible for a majority of voters to prefer $x$ to $y$, $y$ to $z$, and $z$ to $x$, and thus a Condorcet winner, who outperforms each alternative in a pairwise election, may not exist. The result that the pairwise preferences of the aggregate of rational voters may be irrational is known as the Condorcet paradox.

Beyond the applications in social sciences, in more recent times voting has also been used in many interesting problems in computation theory, including rank aggregation in search engines~\cite{DKN,WH} and collaborative filtering~\cite{GNO}. Rank aggregation can also occur when one wishes to select documents or other results based on multiple criteria \cite{LKH}. Some of these applications will be addressed in the next section.

Unfortunately, it is now known that no voting system is perfect, and three major results show the weaknesses inherent in voting. Arrow's impossibility theorem~\cite{Arr} shows that in any election with as few as three candidates, no voting system can convert the ranked preferences of a group of voters into a community-wide ranking meeting a reasonable set of criteria: unrestricted domain, meaning the aggregation is deterministic and complete for all possible profiles of the voters, nondictatorship, meaning that no single voter should be able to single-handedly determine the outcome, Pareto efficiency, meaning that if every voter prefers a candidate to another, this order must also hold in the aggregation, and independence of irrelevant alternatives, meaning that changes to the voters' rankings of irrelevant candidates should not change the outcome for the candidate in question.

Beyond the problems of fairness, election systems are often subject to misuses, where the outcome of the election is unfairly affected by some of the voters, the chair of the election, or outside agents whom attempt to influence the voters. One form of misuse is when a voter reports preferences that are insincere to his or her true preferences for his or her own benefit, or when a group of voters conspire to collectively affect the outcome of the election. This is called manipulation, and is often disastrous for the system as a whole. In particular, a common problem in many election systems is the design that encourages voters to bury their $2^{nd}$ desired candidate, who may also be popular, among their preference list, giving their favorite candidate a more significant advantage. Done as a whole in the election, this distorts the true preferences of the participants. Two important results showing this weakness are the Gibbard-Satterthwaite theorem~\cite{Sat,Gib} and Duggan-Schwartz theorem~\cite{DS}, which show that eliminating the possibility of manipulation is impossible for any reasonable election system.

The Gibbard-Satterthwaite theorem~\cite{Sat,Gib} shows that every election of three or more candidates that chooses a single winner must either be dictatorial (i.e., where a single individual can choose the winner), nonsurjective (i.e., where some candidate cannot win under any circumstances), or is theoretically vulnerable to tactical voting in some cases. Tactical voting occurs when a voter or coalition of voters with full knowledge of the other voters' preferences has an incentive to vote contrary to their true preferences. Since neither of the first two properties are acceptable in any reasonable election system, this can be interpreted as all elections are subject to manipulation. The Duggan-Schwartz theorem~\cite{DS} shows the same result with elections that elect a nonempty set of winners.

We illustrate the problem of manipulation with the following example. In the election system Veto, the lowest candidate in each preference profile is given one veto, and the candidate with the fewest vetoes wins. Consider a system of three candidates, $x$, $y$, and $z$, and, initially, $100$ voters. $40$ initial voters have preferences expressed by $y \succ z \succ x$, $35$ have preferences $x \succ z \succ y$, and $25$ have preferences $y \succ x \succ z$. A total of $22$ manipulators have sincere preferences expressed by $x \succ y \succ z$. In an honest election, $x$ would receive $40$ vetoes, $y$ $35$ vetoes, and $z$ $47$ vetoes, making $y$ the winner. In this case, the $22$ manipulators can collectively make $x$ the winner if $6$ of the manipulators submit the preference $x \succ z \succ y$. Because $x$ would not be the winner of an honest Veto election, such manipulations may produce a sub-optimal outcome. In this example, the manipulators need to know the preference orderings of the honest voters, which may or may not be true depending on the problem of interest. Unfortunately, all reasonable election systems have cases in which such voting behavior is encouraged.

In a related problem, an outside agent can also convince, or bribe, some of the agents to change their votes~\cite{FHH06}. In the computational problem of bribery, the goal of the briber is to modify the outcome of the election with the least amount of effort. There are many interesting problem models. Most notably, voters may be unpriced or priced. In the unpriced case, we are attempting to find a bribery which changes the preferences of the fewest number of voters, and in the priced case, each voter is given a price tag. In this case, we want to find the cheapest bribery. In other cases, the cost of bribing a voter may depend on how significantly his or her preferences are to be changed~\cite{FHH06} (see also \cite{FHHR07}).

In the above example, we have a total of $122$ unweighted voters, $40$ of which choose $x$ as their least favorite candidate, $35$ choose $y$, and $47$ choose $z$, making $y$ a unique winner. We may make $x$ a winner by convincing three voters whom dislike $x$ to submit preference profile $z \succ x \succ y$ instead. This gives $x$ a total of $37$ vetoes, $y$ $38$ vetoes, and $z$ $47$ vetoes. This is an optimum solution if each voter is equally costly to bribe, as it bribes the fewest voters.

In other cases that we will evaluate, some voters are more costly to bribe than others. For example, in a campaign, it may be noted that a subgroup of voters will take more effort for one to convince. In the example above, depending on the price to bribe each voter, it may be the case that it is less costly to bribe five (six if a unique winner is desired) voters whom dislike $z$, to veto $y$ instead. This gives $x$ a total of $40$ vetoes, $y$ $40$ vetoes, and $z$ $42$ vetoes.

It is further possible for the chair of an election to control the outcome of an election by manipulating the set of voters or candidates that will be involved~\cite{BTT92}. One such way to control is by encouraging or discouraging potential voters from participating. One prominent example in politics occurred in 1971 when President Nixon signed the $26^{th}$ amendment into the Constitution of the United States, lowering the legal voting age from 21 to 18, and thus adding a set of voters to the election. The chair can also partition the set of voters to modify the outcome of the election, similar to what is seen with electoral colleges in the United States. As often observed in the presidential elections, the winner of the popular (i.e., plurality) vote need not be the winner of the electoral vote. A related problem, in which some candidates are cloned (i.e., a new candidate with similar properties is introduced to split the voters of this candidate), has also been evaluated \cite{EFS} for several systems, including $k$-Approval. It is also possible to partition the set of candidates, as is done with the candidates of the Republican and Democratic parties, as well as to add or remove candidates.

In practice, an agent wishing to accomplish its goals will likely resort to a combination of more than one of the above strategies. Such problems are known computationally as multimode attacks~\cite{FHH,FHH11}, and will not be covered in this paper.

Bartholdi, Tovey, and Trick~\cite{BTT89} challenged the impossibility of avoiding election misuse by making the observation that manipulation only constitutes a threat when it is computationally feasible to determine how one may manipulate the election for a given instance in the system of interest. They considered an election to be computationally resistant to manipulation if determining such an exploit is NP-hard, and computationally vulnerable if it is polynomial-time computable, and proved a number of P and NP-hard results for manipulating some election systems with one strategic voter. Since then, many results have characterized this complexity result. For example, Conitzer, Lang, and Sandholm~\cite{CLSjour} evaluate the model of manipulation by a coalition of voters. In this model, which we also use in this chapter, the coalition of voters are working together to affect the outcome of the election.

The line of reasoning of Conitzer, Lang, and Sandholm, of making coalitional manipulation at least NP-hard, was applied to bribery in~\cite{FHH06}. Hemaspaandra and Hemaspaandra~\cite{HH07} characterize these hardness results for the case of scoring protocols, by showing exactly which scoring protocols are computationally resistant or vulnerable to manipulation.

There also exist other models of interest for the problem of election manipulation, which we will not review in this chapter. In some models, each manipulator is not aware of the existence or actions of other potential manipulators. Models of manipulation under limited communications in these cases include that of safe manipulation \cite{SW08} (see also \cite{IYEW}). In this case, the manipulator wishes to vote in a way such that the election never produces an undesirable outcome. Other possible models involve probabilistic measures of susceptibility to manipulation \cite{RPW}.

We examine an infinite set of election systems, approval-based 
families of scoring protocols, where each candidate approves of some function $0 \leq f(m) \leq m$ of the $m$ candidates in the election, evaluating the worst-case complexities of various forms of attack. Essentially, we make the first attempt at extending the work of~\cite{HH07,BTT92} to scoring protocols of an unbounded number of candidates, by looking at an infinite set of election systems. A characterization of an infinite set of election systems is also evaluated in \cite{FHS08}.

We find that there are generally only a few cases for the complexity of manipulation in such election systems: Nearly all problems of misuses of election systems involving controlling the voter set are hard by reductions from Set-Cover-type problems, while some problems involving controlling the candidate set are hard by a reduction from Hitting Set. Other cases are easy by greedy algorithms or variations of the Edge Cover problem. We hope that from our work one can gain a better understanding of what properties of elections make them computationally resistant or vulnerable to manipulation, so that our work can be extended to more general forms of elections.

\section{Table of Results}

We summarize the worst-case complexity results of this chapter in the tables below, with new results in bold.\footnote{In the table, swb is the complexity of Simple Weighted $b$-Edge Cover of Multigraphs and sbw that of Simple $b$-Edge Weighted Cover of Multigraphs.}

\subsubsection{Unweighted Cases}

\noindent \begin{tabular}{ l || c | c | c | c | }

& 1-app & 2-app & 3-app & $k$-app, $k \geq 4$  \\

\hline

\hline

Constructive Manipulation                   & P    & P    & P    & P \\

\hline

Constructive Bribery                        & P    & \textbf{P}    & \textbf{NPC}  & \textbf{NPC}  \\

\hline

Constructive Control by Adding Voters       & P    & \textbf{P}    & \textbf{P}    & \textbf{NPC} \\

\hline

Constructive Control by Deleting Voters     & P    & \textbf{P}    & \textbf{NPC}  & \textbf{NPC} \\

\hline

Constructive Control by Adding Candidates   & NPC  & \textbf{NPC}  & \textbf{NPC}  & \textbf{NPC} \\

\hline

Constructive Control by Deleting Candidates & NPC  & \textbf{NPC}  & \textbf{NPC}  & \textbf{NPC} \\

\hline

\end{tabular}

\vspace{10 mm}

\noindent \begin{tabular}{ l || c | c | c | c | }

& 1-veto & 2-veto & 3-veto & $k$-veto, $k \geq 4$ \\

\hline

\hline

Constructive Manipulation                   & P    & P    & P    & P \\

\hline

Constructive Bribery                        & P    & \textbf{P}    & \textbf{P}    & \textbf{NPC} \\

\hline

Constructive Control by Adding Voters      &  P    & \textbf{P}    & \textbf{NPC}  & \textbf{NPC}\\

\hline

Constructive Control by Deleting Voters     & P    & \textbf{P}    & \textbf{P}    & \textbf{NPC}\\

\hline

Constructive Control by Adding Candidates   & NPC  & \textbf{NPC}  & \textbf{NPC}  & \textbf{NPC}\\

\hline

Constructive Control by Deleting Candidates  & NPC  & \textbf{NPC}  & \textbf{NPC}  & \textbf{NPC} \\

\hline

\end{tabular}

\subsubsection{Weighted Voter Cases}

\noindent \begin{tabular}{ l || c | c | c | c |}

& 1-app & 2-app & 3-app & $k$-app, $k \geq 4$\\

\hline

\hline

Constructive Manipulation                   & P    & NPC  & NPC  & NPC   \\

\hline

Constructive Bribery                        & P  & NPC  & NPC  & NPC \\

\hline

Constructive Control by Adding Voters       & P    & ?    & ?    & \textbf{NPC} \\

\hline

Constructive Control by Deleting Voters     & P    & \textbf{sbw}    & \textbf{NPC}  & \textbf{NPC} \\

\hline

Constructive Control by Adding Candidates   & NPC  & \textbf{NPC}  & \textbf{NPC}  & \textbf{NPC} \\

\hline

Constructive Control by Deleting Candidates   & NPC  & \textbf{NPC}  & \textbf{NPC}  & \textbf{NPC} \\

\hline

\end{tabular}

\vspace{10 mm}

\noindent \begin{tabular}{ l || c | c | c | c |}

& 1-veto & 2-veto & 3-veto & $k$-veto, $k \geq 4$\\

\hline

\hline

Constructive Manipulation                   & NPC  & NPC  & NPC  & NPC \\

\hline

Constructive Bribery                 & NPC  & NPC  & NPC  & NPC\\

\hline

Constructive Control by Adding Voters       & P    & \textbf{sbw}    & \textbf{NPC}  & \textbf{NPC}\\

\hline

Constructive Control by Deleting Voters       & P    & ?    & ?    & \textbf{NPC}\\

\hline

Constructive Control by Adding Candidates       & NPC  & \textbf{NPC}  & \textbf{NPC}  & \textbf{NPC}\\

\hline

Constructive Control by Deleting Candidates       & NPC  & \textbf{NPC}  & \textbf{NPC}  & \textbf{NPC}\\

\hline

\end{tabular}

\subsubsection{Unweighted \$Bribery Cases}

\noindent \begin{tabular}{ c || c | c | c | c | c | c | c | c | }

& 1-app & 2-app & 3-app & $k$-app, $k \geq 4$  \\

\hline

\hline

Constructive \$Bribery                        & P & ? & \textbf{NPC}  & \textbf{NPC} \\

\hline

\end{tabular}

\vspace{10 mm}

\noindent \begin{tabular}{ c || c | c | c | c | c | c | c | c | }

& 1-veto & 2-veto & 3-veto & $k$-veto, $k \geq 4$ \\

\hline

\hline

Constructive \$Bribery                        & P & $\geq_{p}$ \textbf{swb} & ? & \textbf{NPC} \\

\hline

\end{tabular}

%%%%%%%%%%%%%%%%%%%%%%%%%%%%%%%%%%%%%%%%%%%%%%%%%%%%%%%%%%%%%%%%%%%%%%%%%

\section{\uppercase{Preliminaries}}

In this chapter, we define what an election system is, and formally define the problems that may occur in the system: manipulation, control, and bribery. We then define other computational problems, such as Edge Cover and Set Cover, which will be utilized in evaluating the complexity of these election problems. Finally, we will review previous complexity results of these related computational problems.

\section{Definition of Elections}\label{sect:defelect}

An election $E = (C,V)$ is defined as a pair of a set of voters $V = \{v_1,\ldots,v_n\}$ and a set of candidates $C = \{c_1,\ldots,c_m\}$. Each voter $v_i$ has a preference over the candidates. A common model for understanding voter preferences is that of the transitive voter model. In the transitive voter model, which we will use exclusively in this chapter, each preference is a strict linear ordering over the candidates $c_{i_1} \succ \cdots \succ c_{i_m}$. But voter preferences need not be transitive (see, e.g., \cite{FHHR07}, in which the voter submits a reflexive and antisymmetric binary ordering over the candidates $C$). Furthermore, voters in some election systems cannot submit their vote simply as a preference ordering. A prominent example is approval, which we will define below.

Elections may be weighted or unweighted. In a weighted election, each voter $v_i$ has a weight, $w(v_i)$, and their vote is counted as $w(v_i)$ individual votes in an unweighted election with the same preference ordering.

\subsection{Election Systems of Interest}~\label{sect:electsys}

An election system $\mathcal{E}$ specifies how one arrives at the outcome given the collection of voter preferences. Depending on the context, the outcome we may be interested in may be a winner, a nonempty subset of winners, or an aggregate preference ordering. Most typically, an election for us will be what in the literature is called a social choice correspondence, namely, a mapping from the candidates and the voter preferences to a subset (known as the winner set) of the set of candidates. There are many systems of interest, some of which involve multiple rounds. Perhaps the most common election system in everyday use is that of plurality, in which the candidate(s) most frequently ranked first among the voters is elected. Plurality is an election system in a more general family of election systems called scoring protocols.

A \emph{scoring protocol} (see the handbook article \cite{BF02}) is defined over a vector $(\alpha_1,\ldots,\alpha_m) \in {\textrm{Z}}^m$. Each candidate $c$ is given $\alpha_i$ points for each voter that ranks $c$ in the $i^{th}$ position of his or her ranking. The candidate(s) with the highest score wins. In addition to plurality, common scoring protocols include veto, Borda count, and approval-based systems with a fixed number of candidates.

A \emph{family of scoring protocols} is an infinite series of scoring protocols $(\alpha^1,\ldots,\alpha^m,\ldots)$, where $\alpha^m=(\alpha^m_1,\ldots,\alpha^m_m)$ is a scoring protocol of $m$ candidates.

In \emph{plurality}, a common form of elections in political science, each voter gives one point to his or her favorite candidate, whereas in \emph{veto}, each voter approves, or gives one point, to all but one candidate: This has the effect of vetoing one candidate. Plurality and veto can thus be considered a family of scoring protocols of the form $\alpha^m = (1,0,\ldots,0)$ and $\alpha^m = (1,\ldots,1,0)$ respectively.

In \emph{$k$-approval}, each voter gives his or her $k$ favorite candidates $1$ point. Similarly, in \emph{$k$-veto}, each voter vetoes his or her least $k$ favorite candidates.

Although not a scoring protocol, another common election system is \emph{approval}~\cite{BFapp}. In approval voting, each voter can approve as many or as few candidates as he/she chooses. In some models, voters may still have a preference order, along with the number of candidates approved. The winner(s) is thus the candidate(s) with the most number of approvals (or weights thereof). In~\cite{FHH06,HHR07}, it is shown that approval is computationally resistant to many forms of misuse involving the voter set, even with unweighted voters, for constructive and destructive cases. We will in fact examine the properties of this election system that give rise to the resistance.

A generalization of $k$-approval and $k$-veto, \emph{$f(m)$-approval}, where $f$ is a function of the number of candidates $m$, is an election where each voter gives $1$ point to each of his or her $f(m)$ favorite candidates. In this case, we will assume that $f(m)$ may be computed in time polynomial with respect to the quantity $m$. This ensures we may compute the outcome of an instance of this election in polynomial time.

Another interesting scoring protocol is the \emph{Borda count}~\cite{Borda}. For $m$ candidates, this election is the scoring protocol defined by the vector $(m-1,m-2,\ldots,1,0)$. In this case, the first preference of each voter is given $m-1$ points, the second $m-2$ points, and so forth.

\section{Problems of Interest in Elections}

A common problem in all election systems of interest is the incentive for dishonesty. In all election systems of interest, there exist instances in which it is advantageous for some of the voters to vote dishonestly, affecting the outcome of the election to their advantage~\cite{Gib,Sat,DS}. Because manipulation is to be carried out by agents with supposedly limited computational power, Bartholdi, Tovey, and Trick~\cite{BTT89} make the observation that manipulation is only a threat to the integrity of the election when the determination of the manipulation is tractable. It is thus of interest to evaluate the complexity, worst-case or otherwise, of computing how these agents may capitalize on these weaknesses. Two other interesting problems include bribery and control: Elections may be bribed~\cite{FHH06}, in which an outside agent influences the election by affecting the voters, or controlled~\cite{BTT92}, in which the chair of the election affects the election by modifying the set of participating voters and candidates.

Each of these problems come in two flavors: constructive, in which our goal is to ensure that a specified distinguished candidate is a winner, and destructive, where our goal is to ensure that such is not a winner. In manipulation, we attempt to reach this goal by giving preferences to a set of unestablished voters, whereas in bribery we do so by changing a given number of votes. There are two subclasses of control problems, those that alter the voter set (i.e., add or delete voters from the election) and those that alter the candidate set.

We define these problems as follows.

\begin{description}
\item[Name:] \emph{$\mathcal{E}$-Manipulation}~\cite{BTT89,BO91,CLSjour}
\item[Instance:] A set $C$ of candidates, a set $V$ of established voters, and $V'$ of unestablished voters such that $V \cap V' = \emptyset$, and distinguished candidate $p$.
\item[Question (constructive manipulation):] Does there exist an assignment of preferences for $V'$ such that $p$ is a winner of the $\mathcal{E}$ election with candidate set $C$ and voter set $V \cup V'$?
\item[Question (destructive manipulation):] Does there exist an assignment of preferences for $V'$ such that $p$ is not a winner of the $\mathcal{E}$ election with candidate set $C$ and voter set $V \cup V'$?
\end{description}

Another problem of interest is the cases where voters have weights. In such case, we denote the problem by $\mathcal{E}$-weighted-constructive-manipulation and $\mathcal{E}$-weighted-destructive-manipulation.

\begin{description}
\item[Name:] \emph{$\mathcal{E}$-Bribery}~\cite{FHH06}
\item[Instance:] A set $C$ of candidates, a set $V$ of voters, distinguished candidate $p$, and nonnegative integer quota $q$.
\item[Question (constructive bribery):] Is it possible to make $p$ a winner of the $\mathcal{E}$ election by changing the preference profiles of at most $q$ voters in $V$.
\item[Question (destructive bribery):] Is it possible to make $p$ not a winner of the $\mathcal{E}$ election by changing the preference profiles of at most $q$ voters in $V$.
\end{description}

As in the cases of manipulation, bribery is also defined for cases of weighted voters. As above, we denote this problem $\mathcal{E}$-weighted-bribery. In addition, each voter $v$ can be assigned a price tag $\pi(v)$ of a nonnegative integer. In this case, $q$ is our budget, and we want to achieve our goal (constructive or destructive bribery) by spending at most $q$ in our bribery. We denote this problem by $\mathcal{E}$-\$bribery and $\mathcal{E}$-weighted-\$bribery if voters have both weights and prices.
There are six problems of interest in the case of control, as we may add, delete, or partition, either voters or candidates. Also in all of these cases, voters may be weighted.

The cases of constructive control, except the versions used here of adding candidates, are from \cite{BTT92}, while those of destructive control, as well as the versions of constructive control by adding candidates used here, are from~\cite{HHR07}.

\begin{description}
\item[Name:] \emph{$\mathcal{E}$-Control by Adding Voters}
\item[Instance:] A set $C$ of candidates, a set $V$ of established voters, and $V'$ of unestablished voters such that $V \cap V' = \emptyset$, distinguished candidate $p$, and nonnegative integer quota~$q$.
\item[Question (constructive control):] Does there exist a subset $V'' \subseteq V'$ with $||V''|| \leq q$ such that $p$ is a winner of the $\mathcal{E}$ election with candidate set $C$ and voter set $V \cup V''$?
\item[Question (destructive control):] Does there exist a subset $V'' \subseteq V'$ with $||V''|| \leq q$ such that $p$ is not a winner of the $\mathcal{E}$ election with candidate set $C$ and voter set $V \cup V''$?
\end{description}

\begin{description}
\item[Name:] \emph{$\mathcal{E}$-Control by Deleting Voters}
\item[Instance:] A set $C$ of candidates, a set $V$ of established voters, distinguished candidate $p$, and nonnegative integer quota~$q$.
\item[Question (constructive control):] Does there exist a subset $V' \subseteq V$ with $||V'|| \leq q$ such that $p$ is a winner of the $\mathcal{E}$ election with candidate set $C$ and voter set $V - V'$?
\item[Question (destructive control):] Does there exist a subset $V' \subseteq V$ with $||V'|| \leq q$ such that $p$ is not a winner of the $\mathcal{E}$ election with candidate set $C$ and voter set $V - V'$?
\end{description}

\begin{description}
\item[Name:] \emph{$\mathcal{E}$-Control by Partitioning Voters}
\item[Instance:] A set $C$ of candidates, a set $V$ of established voters, and distinguished candidate~$p$.
\item[Question (constructive control):] Does there exist a partition of candidates $V = V_1 \cup V_2$ such that $p$ is a winner of the election with the candidates given by the union of the winners of $E_1 = (C,V_1)$ and $E_2 = (C,V_2)$?
\item[Question (destructive control):] Does there exist a partition of candidates $V = V_1 \cup V_2$ such that $p$ is not a winner of the election with the candidates given by the union of the winners of $E_1 = (C,V_1)$ and $E_2 = (C,V_2)$?
\end{description}

\begin{description}
\item[Name:] \emph{$\mathcal{E}$-Control by Adding Candidates}
\item[Instance:] A set $C$ of candidates, a set $V$ of established voters, and $C'$ of unestablished voters such that $C \cap C' = \emptyset$, distinguished candidate~$p$, and nonnegative integer quota $q$.
\item[Question (constructive control):] Does there exist a subset $C'' \subseteq C'$ with $||C''|| \leq q$ such that $p$ is a winner of the $\mathcal{E}$ election with candidate set $C \cup C''$ and voter set $V$?
\item[Question (destructive control):] Does there exist a subset $C'' \subseteq C'$ with $||C''|| \leq q$ such that $p$ is not a winner of the $\mathcal{E}$ election with candidate set $C \cup C''$ and voter set $V$?
\end{description}

\begin{description}
\item[Name:] \emph{$\mathcal{E}$-Control by Deleting Candidates}
\item[Instance:] A set $C$ of candidates, a set $V$ of established voters, distinguished candidate $p$, and nonnegative integer quota $q$.
\item[Question (constructive control):] Does there exist a subset $C' \subseteq C$ with $||C'|| \leq q$ such that $p$ is a winner of the $\mathcal{E}$ election with candidate set $C - C'$ and voter set $V$?
\item[Question (destructive control):] Does there exist a subset $C' \subseteq C \backslash \{p\}$ with $||C'|| \leq q$ such that $p$ is not a winner of the $\mathcal{E}$ election with candidate set $C - C'$ and voter set $V$? In this case, note that we may not deleted the candidate $p$ in which we are manipulating against.
\end{description}

\begin{description}
\item[Name:] \emph{$\mathcal{E}$-Control by Partitioning Candidates}
\item[Instance:] A set $C$ of candidates, a set $V$ of established voters, and distinguished candidate~$p$.
\item[Question (constructive control):] Does there exist a partition of candidates $C = C_1 \cup C_2$ such that $p$ is a winner of the election with the candidates given by the union of the winners of $E_1 = (C_1,V)$ and $E_2 = (C_2,V)$?
\item[Question (destructive control):] Does there exist a partition of candidates $C = C_1 \cup C_2$ such that $p$ is not a winner of the election with the candidates given by the union of the winners of $E_1 = (C_1,V)$ and $E_2 = (C_2,V)$?
\end{description}

In all of these cases, we say that an election is \emph{computationally vulnerable} to a particular form of misuse, such as manipulation, bribery, or control if the corresponding problem is polynomial-time computable. The election is \emph{computationally resistant} to the misuse if the corresponding problem is NP-hard.

\subsection{Some Important NP-Complete Problems}

Common NP-complete problems of choice for showing NP-hardness of election systems include versions of Set Cover, Knapsack, and Hitting Set.

Exact 3-Set Cover (X3C) is a common restricted version of Set Cover.

\begin{description}
\item[Name:] \emph{Exact Cover by 3-Sets (X3C)}~\cite{Karp} (See also~\cite[problem SP2]{GJ79})
\item[Instance:] A set $S = \{s_1,\ldots,s_{3m}\}$ and subsets $T_1,\ldots,T_n \subseteq S$ such that $||T_i||=3$.
\item[Question:] Does there exist a set $A \subseteq \{1,\ldots,n\}$ such that $||A||=m$ and $\bigcup_{i \in A}{T_i} = S$?
\end{description}

Other variations of restricted set cover type problems also exist, such as Exact $\frac{3}{4}$-Set Cover~\cite{FHS08}, which was used to show the NP-completeness of bribery in the Borda count election~\cite{BFHSS}.

An important problem used in showing complexities of problems involving control by adding or deleting candidates is Hitting Set.

\begin{description}
\item[Name:] \emph{Hitting Set}~\cite{Karp} (See also~\cite[problem SP8]{GJ79})
\item[Instance:] A set $S = \{s_1,\ldots,s_m\}$, $n$ subsets of $S$, $T_1,\ldots,T_n$, and positive integer $1 \leq q \leq m$.
\item[Question:] Does there exist a subset of $q$ elements of $S$, $S'=\{s_{i_1}, \ldots, s_{i_q}\}$, such that $S' \cap T_i \neq \emptyset$ for each $1 \leq i \leq n$? Thus, each subset $T_i$ contains at least one element of $S'$.
\end{description}

Hitting Set can also be viewed as a generalization of Vertex Cover to hypergraphs, where edges connect an arbitrary number of vertices. This problem was used to prove that unweighted plurality elections are computationally resistant to control by adding and deleting candidates in~\cite{BTT92}. We will generalize this result.

\section{Edge Covering and Edge Matching Problems}

We will show in this paper that some problems in some election systems are closely related to problems involving edge coverings. This is especially true for elections that distinguish two candidates from the remaining candidates. Due to this connection, these problems generally have similar complexities to that of some forms of edge covering problems.  An edge cover, as well as the well-known decision problem of Edge Cover, is defined as follows.

\begin{definition}
An \emph{edge cover} of a graph is a set of edges such that every vertex of the graph is incident to at least one edge of the set.
\end{definition}

\begin{description}
\item[Name:] \emph{Edge Cover}~\cite{Karp} (See also~\cite[pages 79, 190]{GJ79})
\item[Instance:] An undirected graph $G=(V,E)$ and positive integer $q$.
\item[Question:] Does there exist an edge cover $C \subseteq E$ for $G$ of size at most $q$?
\end{description}

It is possible to find a smallest edge cover in polynomial time, by finding a maximum matching and extending it greedily so that all vertices are covered. We will use several interesting variations of Edge Cover in this paper, which are defined as follows.

In the variation $b$-Edge Cover, each vertex $v$ is to be covered by a minimum of some number, $b(v)$, of edges. There are several interesting variations of this problem: Each edge can be chosen only once (Simple $b$-Edge Cover), an arbitrary number of times ($b$-Edge Cover), or have a capacity and be chosen up to that many times (Capacitated $b$-Edge Cover) (see~\cite{Sch} sections 34.1, 34.7, and 34.8), all of which are polynomial-time computable~\cite{Pul73,CM78,Gab83,Ans87}. In each case, polynomial-time computability is due to the relationship between edge coverings and edge matchings. More specifically, the size of the minimum edge covering and the size of the maximum edge matching always sum to the total of the $b$-values. The maximum edge matching can then be computed using linear programming.

We further extend this problem to that of multigraphs, defining Simple $b$-Edge Cover of Multigraphs as follows.

\begin{description}
\item[Name:] \emph{Simple $b$-Edge Cover of Multigraphs}
\item[Instance:] An undirected multigraph $G=(V,E)$, a function $b:V \rightarrow {\textbf{\textrm{Z}}^{+}}$ and positive integer $q$.
\item[Question:] Does there exist a subset of at most $q$ edges $C \subseteq E$ such that each vertex $v \in V$ is incident to at least $b(v)$ edges in $C$?
\end{description}

This variation is also polynomial-time computable, as follows.

\begin{theorem}
Simple $b$-Edge Cover for Multigraphs is polynomial-time computable.
\end{theorem}

\begin{proof}
We demonstrate how to compute Simple $b$-Edge Cover of Multigraphs using Capacitated $b$-Edge Cover, which we define as follows.

\begin{description}
\item[Name:] \emph{Capacitated $b$-Edge Cover} (See also~\cite[chapter 34]{Sch})
\item[Instance:] An undirected graph $G=(V,E)$, a function $b:V \rightarrow {\textbf{\textrm{Z}}^{+}}$, a capacity function $c:E \rightarrow {\textbf{\textrm{Z}}}^{+}$, and positive integer $q$.
\item[Question:] Does there exist a multisubset of edges $C \subseteq E$ such that each vertex $v \in V$ is incident to at least $b(v)$ edges in $C$ and such that each $e \in E$ is chosen by $C$ with frequency at most $c(e)$?
\end{description}

Let $G=(V,E)$, $b:V \rightarrow {\textbf{\textrm{Z}}}^{+}$, and $q>0$ be an instance of Simple $b$-Edge Cover of Multigraphs.

We construct an instance of Capacitated $b$-Edge Cover as follows. Let $G'$ be a complete graph of $V$. For each pair $v_1, v_2 \in V$, we set the capacity $c(v_1,v_2)$ to the number of edges between $v_1$ and $v_2$ in $G$.~\end{proof}

Simple $b$-Edge Cover of Multigraphs, and several variations thereof, is of interest in solving some election manipulation problems in which the voters always distinguish exactly two candidates from the remaining candidates. More specifically, in such a construction, the vertices correspond to candidates and the edges to the voters, who distinguish two candidates from the remaining. We introduce two additional versions of Simple $b$-Edge Cover of Multigraphs, which will be of interest to problems involving weights and prices for voters.

\begin{description}
\item[Name:] \emph{Simple Weighted $b$-Edge Cover of Multigraphs}
\item[Instance:] An undirected multigraph $G=(V,E)$, a function $b:V \rightarrow {\textbf{\textrm{Z}}^{+}}$, weight function $w:E \rightarrow {\textbf{\textrm{Z}}^{+}}$ and positive integer $q$.
\item[Question:] Does there exist a subset of edges $C \subseteq E$ of total weight at most $q$ such that each vertex $v \in V$ is incident to at least $b(v)$ edges in $C$?
\end{description}

\begin{description}
\item[Name:] \emph{Simple $b$-Edge Weighted Cover of Multigraphs}
\item[Instance:] An undirected multigraph $G=(V,E)$, a function $b:V \rightarrow {\textbf{\textrm{Z}}^{+}}$, weight function $w:E \rightarrow {\textbf{\textrm{Z}}^{+}}$ and positive integer $q$.
\item[Question:] Does there exist a subset of at most $q$ edges $C \subseteq E$ such that each vertex $v \in V$ is incident to edges in $C$ of at least total weight $b(v)$?
\end{description}

In these two versions, as we will see later, the weighted edges correspond to cases of bribery in which each voter has an associated price. The weighted coverings correspond to manipulations involving weighted voters. It is currently unknown whether either of these variations are polynomial-time computable, NP-intermediate, or NP-complete. We will show in this paper that some problems of interest are polynomial-time equivalent to one of these problems.

A related problem involves edge matchings. See~\cite{Sch} for the connection between matchings and coverings, and the extension of Edge Matching to $b$-Edge Matching, and the polynomial-time computability result. We will define an analogous problem for multigraphs as follows.

\begin{description}
\item[Name:] \emph{Simple $b$-Edge Matching of Multigraphs}
\item[Instance:] An undirected multigraph $G=(V,E)$, a function $b:V \rightarrow {\textbf{\textrm{Z}}^{+}}$ and positive integer $q$.
\item[Question:] Does there exist a subset of at least $q$ edges $C \subseteq E$ such that each vertex $v \in V$ is incident to at most $b(v)$ edges in $C$?
\end{description}

This is an extension of Edge Matching. In a traditional Edge Matching problem, each vertex has a $b$-value of $b(v)=1$, as a matching cannot cover a vertex more than once.

\section{Misuses in Elections with a Fixed Number of Candidates}

In~\cite{FHH06}, it is shown that bribery of unweighted scoring protocol elections can be solved in polynomial time with brute force when the number of candidates is fixed, since the number of possible preference orderings is also fixed, and that voters with equal preferences are not distinguishable (see the definition of succinct preference representation in Section~\ref{sect:succinct}). The same principle also applies to control by adding or deleting voters, as well as manipulation. Bribery of fixed scoring protocols is also polynomial-time computable even when voters have prices, since it suffices to bribe only the cheapest voters among those with equal preferences. We review the result for \$bribery from~\cite{FHH06} below.

\begin{theorem}~\cite{FHH06}
$\alpha$-\$bribery is polynomial-time computable for all scoring protocols $\alpha = (\alpha_1,\ldots,\alpha_m)$ with a fixed number of candidates $m$.
\end{theorem}
\begin{proof}
Because $m$ is a constant, there are only a constant, $m!$, number of distinct preference profiles. The algorithm given in~\cite{FHH06} partitions the set of voters into $m!$ subsets, $V = V_1 \cup \cdots \cup V_{m!}$, with voters in each subset having equal preferences. A bribery can then be described by two sequences of integers, $b_1,\ldots,b_{m!}$, and $d_1,\ldots,d_{m!}$, where each $b_i$ indicates how many voters from $V_i$ we are bribing and $d_i$ indicates how many voters will be given the preferences of $V_i$. Without loss of generality, we can assume we are bribing the cheapest $b_i$ voters of the set $V_i$.

Since each bribery can be described with $2m!$ variables, and each variable is bounded $0 \leq b_i,d_i \leq ||V||$, this algorithm tries at most $(||V||+1)^{2m!}$ possible briberies. This quantity is polynomial with respect to the size of the problem instance. Thus, $\alpha$-\$bribery is polynomial-time computable.~\end{proof}

Also owing to the constant number of possible preference orderings, unweighted scoring protocols of a fixed number of candidates are also computationally vulnerable to manipulation as well as control by adding or deleting voters.

\begin{theorem}
All scoring protocols $\alpha = (\alpha_1,\ldots,\alpha_m)$ are computationally vulnerable to constructive control by either adding or deleting voters, as well as to manipulation for a fixed number of candidates $m$.
\end{theorem}

It is important to note that despite the theoretical polynomial-time computability in this case, in practice, this problem is significantly difficult for moderate-sized $m$, due to the exponential nature of $m!$. It is often the case that the brute-force algorithm, despite being faster in the worst case, is not the best approach to this problem empirically.

In~\cite{HH07,FHH06} it is shown that most nontrivial manipulation and bribery problems in weighted scoring protocol elections are NP-hard even for a fixed number of candidates. We will discuss these cases in Section~\ref{sect:weightprice}.

It should be noted that control by adding or deleting candidates in elections of a constant number of candidates (participating and additional candidates) is easy, as there are a constant number of different sets of candidates to add or delete, and whether or not a distinguished candidate will win or lose the election can be checked in polynomial time.

\section{Destructive Misuses}
To demonstrate the vast differences between the problems of constructive misuse and that of destructive misuse, we prove that unweighted approval-based elections are computationally vulnerable to most forms of destructive misuse. This is true even for families of scoring protocols, in which the size of the candidate set is unbounded. Consider the problem of bribery below.

\begin{theorem}
Unweighted $f(m)$-approval elections are computationally vulnerable to destructive bribery.
\end{theorem}
\begin{proof}
As a reminder, we assume that $f(m)$ is computable in time polynomial with respect to $m$.

In a destructive bribery, we want some candidate $p' \neq p$ to beat $p$. Thus, we want to have $p'$ gain as many points as possible relative to $p$.

For each $p' \neq p$, we separate the voters into 3 categories:

\noindent 1. Those approving of $p$ but not of $p'$.

\noindent 2. Those approving either both of $p$ and $p'$ or neither of $p$ nor $p'$.

\noindent 3. Those approving of $p'$ but not of $p$.

Clearly, briberies of votes of the first type take priority over those of the second type, as $p'$ can gain two points relative to $p$. Voters of the third type should not be bribed at all. We check whether $p'$ can potentially beat $p$ by first bribing voters of type one, and then those of type two, up to a total of our quota $q$. We always bribe voters to approve of $p'$ but not $p$.

We find such a bribery for each other candidate $p' \neq p$ by bribing $q$ voters using the procedure above. If there exists no candidate $p' \neq p$ that can be made to beat $p$, we reject.~\end{proof}

Intuitively, the reason such briberies are easy is that we only need to ensure that $p'$ beats $p$. It does not matter if other candidates will beat both $p'$ and $p$ as a result of the bribery, as $p$ will lose. This principle is also demonstrated in~\cite{Rus} for the case of Borda elections and~\cite{FHH06} for the case of scoring protocols in general.

\begin{theorem}
Unweighted $f(m)$-approval elections are computationally vulnerable to destructive manipulation and destructive control adding or deleting voters.
\end{theorem}
\begin{proof}
The concept of attempting to make each candidate $p' \neq p$ beat $p$ can be applied to manipulation, as well as control by adding and deleting voters. In manipulation, each manipulator approves of $p'$ but not of $p$. Similarly, in control by adding voters, we only add voters that approve of $p'$ but not of $p$. In control by deleting voters, we only delete voters that approve of $p$ but not of $p'$.~\end{proof}

We conclude that all unweighted approval-based scoring protocols, as well as families of scoring protocols, $f(m)$-approval, are computationally vulnerable to destructive bribery, manipulation, and control by adding or deleting voters. The case of destructive control by adding or deleting candidates will be considered in Section~\ref{sect:contbycand}. The remainder of this section will thus focus on constructive problems. 

\section{Manipulation of Approval-Based Scoring Protocols}

It is shown in~\cite{Z} that manipulation in the $f(m)$-approval families of scoring protocols (recall that $f(m)$ is computable in polynomial time with respect to the quantity of $m$ in this case) is polynomial-time computable, by a greedy algorithm. The greedy algorithm operates by iteratively assigning the manipulators' preferences as follows: each manipulator approves $p$, the distinguished candidate and the current $f(m)-1$ candidates with the lowest scores. It can be shown that this algorithm is correct by induction on the number of manipulators. Such elections are thus computationally vulnerable to insincere voting from a coalition of agents for their collective benefit, whether their goal is to elect a preferred candidate (constructive manipulation) or to prevent the election of a despised candidate (destructive manipulation).

\section{Bribery in Approval-Based Scoring Protocols}
In unweighted and unpriced families of scoring protocols, the goal of the bribery problem is to determine the minimum number of voters one needs to change the preferences for (i.e., bribe) to achieve a desirable outcome. As we have seen earlier, if our goal is to exclude a candidate from winning, as in destructive bribery, this problem is easy and can be computed in polynomial time for all approval-based elections. We thus now focus on constructive bribery cases, where we aim to elect a distinguished candidate.

\begin{theorem}[\cite{FHH06}]
Unweighted $1$-approval elections are computationally vulnerable to bribery by unpriced voters.
\end{theorem}

\begin{proof}
In this case, one simply bribes the voters that vote for the current winner to vote for our distinguished candidate, $p$, until either $p$ wins, in which case we accept, or until we run out of voters to bribe, in which case we reject.~\end{proof}

\begin{theorem}[\cite{FHH06}]
Unweighted $1$-veto elections are computationally vulnerable to bribery by unpriced voters.
\end{theorem}

\begin{proof}
In this case, we simply bribe voters whom veto $p$ to veto for the current winner instead.~\end{proof}

We now demonstrate how the seemingly unrelated problem of Set Cover relates to bribery of election systems, and prove the resistance to bribery of some approval-based election systems.

In~\cite{HHR07} a reduction is given from X3C to control in approval elections by adding or deleting voters. The reduction operates by encoding the elements as candidates and voters as subsets of elements, by approving these elements. A number of buffer candidates and voters are added to enforce the constraints of the problem of X3C.

We demonstrate how to modify this result to the cases of $k$-approval and $k$-veto, where either the number of candidates approved or vetoed is fixed, for appropriately valued $k$. The reduction requires the usage of additional buffer candidates and voters, as in the reduction given in~\cite{HHR07}, each voter approves a different number of candidates.

\begin{theorem}
Unweighted $k$-approval elections are computationally resistant to bribery by unpriced voters for $k \geq 3$.
\end{theorem}
\begin{proof}
We first show resistance for the case of $3$-approval elections. Cases for $k > 3$ can be shown by a simple change. We make our reduction from Exact $3$-Set Cover (X3C).

Let 
\[S = \{s_1,\ldots,s_{3m}\}\]
 and 
\[T_1 = \{t_{1,1},t_{1,2},t_{1,3}\},\ldots,T_n = \{t_{n,1},t_{n,2},t_{n,3}\}\]
be an instance of X3C. Without loss of generality, we assume $n \geq m$.

We construct our $k$-approval election as follows. Our candidate set will be \[ C = \{p, p', p''\} \cup \{s_1, \ldots, s_{3m}\} \cup \{b_1, \ldots, b_{\frac{3(nm+m-n)}{2}}\}. \]
For each $3$-set $T_i = \{t_{i,1},t_{i,2},t_{i,3}\}$ we construct a voter who approves of $\{t_{i,1},t_{i,2},t_{i,3}\}$. For $1 \leq i \leq \frac{3(nm+m-n)}{2}$ we add voters that approve of $b_i$ and two members of $s_1,\ldots,s_{3m}$. We choose these votes in a way such that each $s$ candidate receives exactly $n+1$ approvals in total. Finally, we add $n-m$ votes that approve of $\{p, p', p''\}$. We set our bribery quota $q$ to $m$.

In this election, we note that each of $p, p', p''$ receives $n-m$ approvals. Each of $s_1$, $\ldots$, $s_{3m}$ receives $n+1$ approvals, and each candidate $b_i$ for $1 \leq i \leq \frac{3(nm+m-n)}{2}$ receives one approval.

If there is an exact covering by $3$-sets, we note that bribing the agents corresponding to the $3$-sets to approve of $\{p, p', p''\}$ instead will allow $p$ to win. On the other hand, a valid bribery must remove one approval of each of $s_1$, $\ldots$, $s_{3m}$, since the most number of approvals that can be given to $p$ is $m$. Note that bribing any of the votes corresponding to $T_i$ removes three approvals, collectively, from $s_1, \ldots, s_{3m}$. In contrast, all other bribes will remove at most two approvals, collectively, from $s_1, \ldots, s_{3m}$. Thus, a successful bribery of $\leq m$ votes can only involve votes corresponding to $T_i$, and each candidate of $s_1, \ldots, s_{3m}$ must lose one approval. This corresponds to an exact covering.

A simple extension to show that $k$-approval elections for $k > 3$ are also computationally resistant to bribery can be constructed by adding additional buffer candidates, such that each voter approves of distinct buffer candidates in addition to the three edges elements in the subset. Each buffer candidate would thus receive only one approval, and cannot influence the election.~\end{proof}

A similar reduction, in the next theorem shows this hardness result for $k$-veto elections.

\begin{theorem}
Unweighted $k$-veto elections are computationally resistant to bribery by unpriced voters for $k \geq 4$.
\end{theorem}
\begin{proof}
Again, we start with $4$-veto elections and demonstrate how our reduction can be extended.

Let \[S = \{s_1,\ldots,s_{3m}\}\]
and
\[T_1 = \{t_{1,1},t_{1,2},t_{1,3}\},\ldots,T_n = \{t_{n,1},t_{n,2},t_{n,3}\}\] be an instance of X3C. In this reduction, we will restrict the X3C instance to cases in which $n \geq 2m$, for reasons which will be clear later.

We construct a $4$-veto election of $3m+5$ candidates, \[C = \{p\} \cup \{b_1, b_2, b_3, b_4\} \cup \{s_1,\ldots,s_{3m}\}.\]

For each set $T_i = \{t_{i,1},t_{i,2},t_{i,3}\}$, we have one voter who veto $\{p, t_{i,1}, t_{i,2}, t_{i,3}\}$. We add $\frac{mn-m^2+m-3n}{4}$ voters whom veto four candidates of $\{s_1,\ldots,s_{3m}\}$ such that each $s_i$ receives exactly $n-m+1$ vetoes. Initially, $n-2m$ voters veto the four buffer candidates, $\{b_1,b_2,b_3,b_4\}$. Under this construction, $p$ initially has $n$ vetoes, each $s_i$ has $n-m+1$ vetoes, and each buffer candidate has $n-2m$ vetoes. We set our quota for bribing voters to $m$.

Consider an exact covering of $S$. Bribing the votes corresponding to the sets to vote for the $4$ buffer candidates instead will leave $p$ with $n-m$ vetoes, each $s_i$ with $n-m$ vetos, and each buffer candidate with $n-m$ vetoes. $p$ is clearly made a winner. Conversely, consider any bribery of at most $m$ voters such that $p$ is made a winner. In any bribery of at most $m$ voters, $p$ must retain at least $n-m$ vetoes. If $p$ is made a winner, each $s_i$ must retain at least $n-m$ vetoes, and can lose at most one veto. Since there are only $3m$ such candidates and only votes corresponding to the 3-sets veto 3 such candidates, these are the only votes we can bribe in such a bribery. Thus, this corresponds to an exact covering.

This reduction can also be extended to $k$-veto elections for all $k>4$, by adding $k$ buffer candidates in total.~\end{proof}

In the next result, we see how a slightly modified form of the greedy algorithm introduced in~\cite{FHH06} can be used to bribe $2$-veto elections.

\begin{theorem}
Unweighted $2$-veto elections are computationally vulnerable to bribery by unpriced voters.
\end{theorem}

\begin{proof}
Consider an election $E=(C,V)$ and quota $q > 0$.

Without loss of generality, we will only bribe voters that veto $p$. If no voters veto $p$, then $p$ is a winner and we can accept. Similarly, if $p$ has at most $q$ vetoes, then we can bribe these voters and make $p$ a winner. Thus, we will assume that $p$ has more than $q$ vetoes. Also without loss of generality, we will not bribe any voters to veto $p$.

For each candidate $c \in C$, let $\mathrm{Vetoes}(c)$ be the number of vetos currently received by candidate $c$, and for $c_1 \neq c_2$, $\mathrm{Vetoes}(c_1,c_2)$ be the number of voters who veto both $c_1$ and $c_2$. Then, following our bribery of $q$ voters, $p$ will receive exactly $\mathrm{Vetoes}(p)-q$ vetoes. Since we don't want any candidate to have fewer than this many vetoes, for each $c \neq p$, we can bribe up to $\mathrm{max}(\mathrm{Vetoes}(p,c),\mathrm{Vetoes}(p)-\mathrm{Vetoes}(c)-q)$ voters whom veto both $p$ and $c$ and ensure that $c$ will not beat $p$.

We will thus bribe a total of \[\mathrm{max}(\sum_{c \neq p} {\mathrm{max}(\mathrm{Vetoes}(p,c),\mathrm{Vetoes}(p)-\mathrm{Vetoes}(c)-q)}, q)\] voters. The votes chosen does not matter, since we are removing vetoes from candidates who will not beat $p$, and we are bribing the maximum number of votes feasible. We must then assign these votes such that each candidate currently beating $p$ receives enough vetoes to lose to $p$. This is similar to the polynomial-time computable problem of $2$-veto manipulation.~\end{proof}

The next two results, of bribery of $2$-approval and $3$-veto elections, demonstrate the connection between the seemingly unrelated problem of Simple $b$-Edge Cover and manipulating election systems. We will see in this and other results that this connection occurs most often when the election system distinguished two candidates from the remaining candidates. This is because the connection equates candidates to vertices of graphs and voters to edges, with the two candidates being distinguished by each voter equated with the two vertices connected by the corresponding edge.

\begin{theorem}
Unweighted $2$-approval elections are computationally vulnerable to bribery by unpriced voters.
\end{theorem}

\begin{proof}
We solve $2$-approval bribery using Simple $b$-Edge Cover.

Consider an election $E=(C,V)$. We wish to ensure the victory of $p$ by bribing at most $q$ voters.

We first observe that, without loss of generality, we will only bribe voters that do not approve of $p$. If there are no more than $q$ such voters, then all voters will approve of $p$ and $p$ will win. Thus, without loss of generality, we assume that more than $q$ voters do not approve of $p$. Also without loss of generality, we can also assume that we give $p$ one approval for each of these bribed votes, and that we will bribe exactly $q$ voters.

For each candidate $c \in C$, let $\mathrm{app}(c)$ be the number of voters currently approving candidate $c$, and for $c_1 \neq c_2$, $\mathrm{app}(c_1,c_2)$ be the number of voters who approve both $c_1$ and $c_2$. Following bribery, $p$ will receive $\mathrm{app}(p) + q$ approvals. Further, for each candidate $c \neq p$, define $\mathrm{def}(c) = \mathrm{app}(p) + q - \mathrm{app}(c)$. If $\mathrm{def}(c) > 0$, we can give $\mathrm{def}(c)$ approvals to candidate $c$. If $\mathrm{def}(c) < 0$, we must bribe $- \mathrm{def}(c)$ voters who currently approve of $c$.

Let $C_1 = \{v\ |\ \mathrm{def}(c) \geq 0\}$ and $C_2 = \{v\ |\ \mathrm{def}(c) < 0\}$. Thus, for each $c \in C_2$, we must bribe at least $- \mathrm{def}(c)$ voters who currently approve of $c$. Suppose that we will bribe exactly $s_1$ voters approving one voter in $C_2$ and $s_2$ voters approving two voters in $C_2$. We will explain what $s_1$ and $s_2$ will be later.

Let $E = \sum_{c \neq p} {\mathrm{max}(0,-\mathrm{def}(c))}$. This is the total number of ``excess'' approvals that must be bribed from each nondistinguished candidates. If $E > 2q$, then clearly bribery is not possible, since a bribery of $q$ voters can only remove $2q$ excess approvals. Conversely, let $D = \sum_{c \neq p} {\mathrm{max}(0,\mathrm{def}(c))} + s_1$. This is the total ``deficit'' approvals that we can give to the nondistinguished candidates; initially, we can give $\sum_{c \neq p} {\mathrm{max}(0,\mathrm{def}(c))}$ approvals to the candidates of $C_1$. We can give $s_1$ more approvals due to the bribery of voters approving candidates in $C_1$.

Thus, we can bribe no more than $D$ voters, since each voter will have to approve one nondistinguished candidate in addition to $p$ (Note that without loss of generality, we will not bribe voters who have a positive deficit, and thus, at most $D$ approvals may be given to those candidates). Thus if $q < D$, then bribery is not possible with this set of parameters.

Construct the multigraph $G$ as follows.

Let $V(G) = C_2 \cup \{x\}$ and define $E(G)$ as follows. For every voter approving $\{u,v\} \subseteq C_2$, we add an edge $(u,v)$ and for each voter approving $u \in C_2$ and $v \in C_1$, we add an edge $(u,x)$. We set the $b$-values to $b_c = -\mathrm{def}(c)$ for $c \in C_2$ and $b_x = s_1$. Note that in essence, $x$ represents the candidates of $C_1$ in this construction.

\noindent \textbf{Claim:} Given that $D \leq q \leq \frac{E}{2}$, $G$ has a simple $b$-edge covering of $q$ edges iff there exists a bribery of $q$ voters, of which $s_1$ of the bribed voters approve one candidate of $C_2$ and $s_2$ approve two candidates of $C_2$, making $p$ a winner.

Consider a simple $b$-edge covering of $q$ edges. Without loss of generality, we may assume $x$ is covered by exactly $b_x = s_1$ edges, as we may modify the covering otherwise. We bribe the voters corresponding to the $q$ edges to approve of $p$ and one nondistinguished candidate. In the case of edges linking $v \in G$ and $x$, it suffices to bribe any voter approving $v$ and a candidate in $C_1$. Since $q \leq D$, it is possible to give the latter approval to the candidates in $C_1$ without exceeding the score of $p$. As $b(v)$ corresponds to the number of approvals that must be removed from each $v \in C_2$, $p$ will beat each candidate in $C_2$.

Conversely, consider a bribery of voters electing $p$ satisfying the paramters. For each candidate $c \in C_2$, at least $-\mathrm{def}(c)$ voters approving $c$ must be bribed. This corresponds to a $b$-edge cover. Also, by the definition of the parameters, exactly $s_1$ voters approving candidates in $C_1$ will be bribed. This corresponds to a solution of the b-Edge Cover problem above.

To find a bribery electing $p$, we enumerate the values of $s_1$ and $s_2$ such that $s_1 + s_2 = q$ and attempt the corresponding construction of $b$-Edge Cover.~\end{proof}

\begin{theorem}
Unweighted $3$-veto elections are computationally vulnerable to bribery by unpriced voters.
\end{theorem}
\begin{proof}
We observe that, instead of bribing voters not approving $p$, we bribe voters vetoing $p$. A similar reduction from this problem to that of $b$-Edge Cover can be constructed by converting votes vetoing candidates $p, c_1,$ and $c_2$ into an edge between $c_1$ and $c_2$ in our graph.~\end{proof}

\section{Controlling an Election via Voters}
The chair of an election can control the outcome by controlling the candidate set. Three methods of candidate control of interest are, adding voters to the voter set, removing voters from the voter set, and partitioning the voters and conducting subelections. An example of the latter case can be seen in the Electoral Colleges in the Presidential Elections of the United States, where a separate plurality election is conducted within each state. Throughout the history of the United States, minorities, women, and persons under the age of $21$ have also been added to the voter set of the Presidential Election.

We evaluate the computational complexity of such controls in families of approval-based scoring protocols.

\begin{theorem}
Unweighted $1$-approval (i.e., plurality) and $1$-veto (i.e., veto) are computationally vulnerable to control by adding or deleting voters.
\end{theorem}
\begin{proof}
All of these problems can be solved in polynomial time by simple greedy algorithms. To control by adding voters for plurality elections, we simply add votes that approve $p$ until either $p$ wins, in which case we accept, or we are out of votes to add or have added our quota, in which we case reject, as it cannot benefit $p$ to add votes approving another candidate. To control by deleting voters, for each candidate $p'$ beating $p$, we must delete votes approving $p'$ until $p'$ has as many approvals as $p$. In veto elections, to control by adding voters, for each candidate $p'$ beating $p$, we must add votes vetoing $p'$ until it has as many vetoes as $p$. To control by deleting voters, we delete voters vetoing $p$.~\end{proof}

\begin{theorem}
Unweighted $2$-approval elections are computationally vulnerable to control by adding voters, and unweighted $2$-veto elections by deleting voters.
\end{theorem}
\begin{proof}
Consider the case of control by adding voters in $2$-approval elections. Without loss of generality, we only need to consider adding voters which approve of $p$, and consider adding as many voters as possible. This determines the final score of $p$, allowing one to determine if it is possible to choose votes such that no other candidates will exceed this total.

In the case of $2$-veto, we only need to consider deleting voters who veto $p$. We compute the number of vetoes $p$ will retain, and determine if it is possible to delete these voters such that no other candidate will have fewer vetoes.~\end{proof}

\begin{theorem}
Unweighted $3$-approval and $2$-veto elections are computationally vulnerable to control by adding voters, and unweighted $2$-approval and $3$-veto elections by deleting voters.
\end{theorem}
\begin{proof}
As demonstrated in the case of bribery, we represent the candidates as vertices and the voters by edges. The reduction is to Simple $b$-Edge Matching of Multigraphs, in this case.

Consider the case of control by adding at most $q$ unweighted voters in $3$-approval elections. Without loss of generality, we will add exactly $q$ voters, all of which approve $p$. The final score of $p$ is thus $s(p) + q$, where $s(p)$ is the initial score of $p$. For each candidate $c \in p$, we may add at most $s(p) + q - s(c)$ voters approving $c$.

Each voter that is added also approves two other candidates. In this case, each voter corresponds to an edge between the two other candidates, and we must find $q$ edges such that each vertex corresponding to the candidate $c$ is covered by at most $s(p) + q - s(c)$ edges, giving us $q$ voters that can be added without exceeding this score. This corresponds to the problem of Simple $b$-Edge Matching of Multigraphs, which is polynomial-time computable.

A similar reduction will also show that $3$-veto elections are vulnerable to control by deleting voters.~\end{proof}

\begin{theorem}
$k$-approval elections are computationally resistant to constructive control by adding voters for $k \geq 4$.
\end{theorem}

\begin{proof}
We begin by showing that $4$-approval control by adding voters is NP-hard and extend it to $k$-approval for $k > 4$.

Let \[S = \{s_1,\ldots,s_{3m}\}\] and \[T_1 = \{t_{1,1},t_{1,2},t_{1,3}\},\ldots,T_n = \{t_{n,1},t_{n,2},t_{n,3}\}\] be an instance of X3C.

Without loss of generality, we can assume that $4$ divides $3m$, by adding at most $4$ dummy sets to our instance. We define our candidate set \[C = \{p\} \cup \{s_1, \ldots, s_{3m}\}.\] We assign votes to $V$ such that initially $p$ receives no approvals and $s_i$ receives $m-1$ approvals for each $1 \leq i \leq 3m$. Note that this construction is possible provided that $4$ divides $3m$. There are thus $3m+1$ candidates and $\frac{3m(m-1)}{4}$ voters.

For each set $T_i = \{t_{i,1},t_{i,2},t_{i,3}\}$ in our instance of X3C, we add an unestablished voter $v_i$ to $V'$ which approves of $\{p, t_{i,1}, t_{i,2}, t_{i,3}\}$. We show that our instance of X3C has an exact covering if and only if there exists a control of this election by adding $m$ new voters.

Consider an exact $3$-set covering. Adding the votes that correspond to the $m$ $3$-sets of this covering would add one approval to each of $s_i$ and $m$ approvals to $p$, giving each candidate exactly $m$ approvals. Thus, $p$ becomes a winner. Conversely, consider a successful control that adds at most $m$ new voters and makes $p$ a winner. Since each $s_i$ candidate currently has $m-1$ approvals and all votes approve of at least one such candidate, we see that we must add exactly $m$ new voters to potentially make $p$ a winner. However, in such a case, $p$ will end up with $m$ approvals, so none of the $s_i$ candidates can receive two new approvals, as this would give that candidate $m+1$ approvals. This corresponds to a solution to the X3C instance.

In the case of $k$-approval control by adding voters for $k>4$, we modify the reduction as follows. Each voter in $V$ now approves of $k$ candidates in $S$, and thus without loss of generality, we now assume that $k$ divides $3m$. In each voter $v_i$ in $V'$, we will approve of $k-4$ buffer candidates, $\{b_1,\ldots,b_{k-4}\}$ in addition to $\{p, t_{i,1}, t_{i,2}, t_{i,3}\}$. Following an addition of voters, each buffer candidate will have exactly the same number of approvals as $p$, and will this not affect the status of $p$.~\end{proof}

\begin{theorem}
$k$-veto elections are computationally resistant to constructive control by adding voters for $k \geq 3$.
\end{theorem}

\begin{proof}
We begin by showing that $3$-veto control by adding voters is NP-hard and extend it to $k$-veto for $k > 3$.
Let \[S = \{s_1,\ldots,s_{3m}\}\] and \[T_1 = \{t_{1,1},t_{1,2},t_{1,3}\},\ldots,T_n = \{t_{n,1},t_{n,2},t_{n,3}\}\] be an instance of X3C.
We define our candidate set to be \[C = \{p, p', p''\} \cup \{s_1, \ldots, s_{3m}\}.\] Our initial voter set $V$ will consist of one vote, which vetoes $\{p, p', p''\}$.

For each set $T_i = \{t_{i,1},t_{i,2},t_{i,3}\}$, we add a voter $v_i$ to our unestablished voter set $V'$ who vetoes $\{t_{i,1},t_{i,2},t_{i,3}\}$. We show that a manipulation by adding at most $m$ voters exists iff there is an exact covering for this instance of X3C.

Consider an exact $3$-set covering. Adding the $m$ votes corresponding to the covering gives each candidate exactly one veto. $p$ is thus made a winner. In contrast, $p$ can only be made a winner by adding one veto to each candidate in $s_i$. Since only $m$ votes are added, this corresponds to a solution of the X3C instance.~\end{proof}

\begin{theorem}
$k$-approval elections are computationally resistant to constructive control by deleting voters for $k \geq 3$.
\end{theorem}

\begin{proof}
The reduction for this case is similar to that of bribery. In this case, instead of bribing the voters to vote for $\{p, p', p''\}$ we simply delete them.~\end{proof}

\begin{theorem}
$k$-veto elections are computationally resistant to constructive control by deleting voters for $k \geq 4$.
\end{theorem}

\begin{proof}
We note that the construction in the proof that $k$-veto is computationally resistant to bribery in the unweighted and unpriced case also applies to constructive control by deleting voters. In this case, instead of bribing the voters to vote for buffer candidates, we simply delete them.~\end{proof}

\section{Controlling an Election via Candidates}~\label{sect:contbycand}
The chair of an election can also attempt to influence the outcome by affecting the voter set, dictating rules for candidates that may or may not participate in the election. A common example is that of cloning, in which a winning candidate is made worse off by the addition of a similar candidate, causing the votes to be split among them, and possibly allowing a third independent candidate to win.

We briefly review the known complexities of this problem given in~\cite{BTT92} and~\cite{HHR07} and demonstrate how to generalize it to other cases of approval elections. Control by adding and deleting candidates is hard in both the constructive and destructive case in two of the simplest families of scoring protocols: plurality and veto. The reductions given in~\cite{BTT92,HHR07} are from Hitting Set.

\subsection{Destructive Control}
In~\cite{HHR07}, a construction was given demonstrating that destructive control via either adding or deleting candidates in a plurality election is NP-hard, via a reduction from Hitting Set. We demonstrate how to generalize this reduction to the cases of $k$-approval by adding buffer candidates, and then to $k$-veto by making some modifications. This hardness result for adding candidates is also shown independently in \cite{EFS}.

\begin{theorem}
$k$-approval elections are computationally resistant to destructive control by either adding or deleting candidates.
\end{theorem}

\begin{proof}
Consider an instance of Hitting Set, where we are given a set \[S = \{s_1,\ldots,s_m\},\] $n$ subsets of $S$, $T_1,\ldots,T_n$, and positive integer $1 \leq q \leq m$.

In our election, the candidate set will consists of 

\begin{eqnarray*}
C = \{c,c'\} &\cup& S \\
&\cup& \{x^1_{i,j} \ |\  1 \leq i \leq 2(m-q)+2n(q+1)+4, 1 \leq j \leq k\}\\
&\cup& \{x^2_{i,j} \ |\  1 \leq i \leq 2n(q+1)+5, 1 \leq j \leq k\}\\
&\cup& \{x^3_{i,j,\ell} \ |\  1 \leq i \leq n, 1 \leq j \leq 2(q+1), 1 \leq \ell \leq k\}\\
&\cup& \{x^4_{i,j,\ell} \ |\  1 \leq i \leq m, 1 \leq j \leq 2, 1 \leq \ell \leq k\}.\\  
\end{eqnarray*}

For each $1 \leq i \leq 2(m-q) + 2n(q+1) + 4$ we add a voter with preferences \[c \succ x^1_{i,1} \succ \cdots \succ x^1_{i,k-1} \succ c' \succ \cdots,\] and for each $1 \leq i \leq 2n(q+1) + 5$ we add a voter with preferences \[c' \succ x^2_{i,1} \succ \cdots \succ x^2_{i,k-1} \succ c \succ \cdots.\] Thus, there are, just as in~\cite{HHR07}, $2(m-q) + 2n(q+1) + 4$ voters preferring $c$ to all other candidates, and $2n(q+1)+5$ candidates preferring $c'$.

For each $1 \leq i \leq n$ and $1 \leq j \leq 2(q+1)$, we add one voter with preferences \[T_i \succ x^3_{i,j,1} \succ \cdots \succ x^3_{i,j,k-1} \succ c \succ \cdots.\]

For each $1 \leq i \leq m$ and $1 \leq j \leq 2$, we add one voter with preferences \[b_i \succ x^4_{i,j,1} \succ \cdots \succ x^4_{i,j,k-1} \succ c' \succ \cdots.\]

The above construction is similar to that given in~\cite{HHR07}. We note that in this case, we must ensure that we do not delete the buffer candidates in our case for deleting candidates. We thus add the following voters.

For each $1 \leq i \leq 2(m-q) + 2n(q+1) + 4$, one voter with preferences \[x^1_{i,1} \succ \cdots \succ x^1_{i,k} \succ c \succ \cdots\] and for $1 \leq i \leq m$ and $1 \leq j \leq 2$, one voter with preferences \[x^4_{i,j,1} \succ \cdots \succ x^4_{i,j,k} \succ c \succ \cdots.\]

This ensures that $c'$ cannot gain in relation to $c$ by deleting any buffer candidates.

We show that there is an addition of at most $q$ voters to $(\{c,c'\} \cup B,V)$ such that $c$ is excluded from winning if and only if there exists a hitting set of $S$ of at most $q$ elements.
Let $S'$ be a hitting set of $S$ of size $q$. In the election $(S' \cup B \cup \{c,c'\}, V)$, $c$ receives $2(m-k)+2n(q+1)+4$ approvals, $c'$ receives $2(m-k)+2n(q+1)+5$ approvals, each candidate $s_i \in S$ receives at most $2n(q+1)+2$ approvals, and each buffer candidate receives at most $2$ approvals. Thus, $c'$ wins this election and $c$ is excluded from winning.

In contrast, let $D$ be a subset, of size at most $q$, of $S$ such that $c$ is not a winner of $(D \cup B \cup \{c, c'\}, V)$. We first note that if $b \in S$ or $b$ is a buffer candidate, that $c$ beats $b$. Thus, if $c$ is excluded from winning this election, then $c'$ must beat $c$.
In $(D \cup \{c,c'\}, V)$, $c'$ receives $2n(q+1)+5+2(m-||D||)$ approvals and $c$ receives $2(m-q)+2n(q+1)+4+2(q+1)m'$ approvals, where $m'$ is the number of sets in $S$ that are not hit by $D$. Since $c$ is excluded from winning, $2(m-q)+2(q+1)m' \leq 2(m-||D||)$, which implies $(q+1)m' + ||D|| - q \leq 0$. So $m' = 0$ and $||D||$ corresponds to a hitting set of $S$ of size at most $q$.

As in~\cite{HHR07}, this reduction also shows that $k$-approval is computationally resistant to destructive control by deleting voters. In this case, we start with the election $(\{c,c'\} \cup B \cup S, V)$. Destructive control by deleting at most $m-q$ candidates is possible if and only if there exists a hitting set of $S$ of at most $q$ elements. We note here that the additional voters approving the buffer candidates to $c$ is required in this case to prevent one from deleting buffer candidates. In this case, deleting buffer candidates would be useless because $c$ would gain one approval while $c'$ would gain at most one approval.~\end{proof}

\begin{theorem}
$k$-veto elections are computationally resistant to destructive control by either adding or deleting candidates.
\end{theorem}

\begin{proof}
Consider an instance of Hitting Set, where we are given a set $S = \{s_1,\ldots,s_m\}$, $n$ subsets of $S$, $T_1,\ldots,T_n$, and positive integer $1 \leq q \leq m$.

In our election, the candidate set will consists of

\begin{eqnarray*}
C = \{c,c'\} &\cup& S \\
&\cup& \{x^1_{i,j}\ |\  1 \leq i \leq 2(m-q)+2n(q+1)+4, 1 \leq j \leq k\}\\
&\cup& \{x^2_{i,j}\ |\  1 \leq i \leq 2n(q+1)+5, 1 \leq j \leq k\}\\
&\cup& \{x^3_{i,j,\ell}\ |\  1 \leq i \leq n, 1 \leq j \leq 2(q+1), 1 \leq \ell \leq k\}\\
&\cup& \{x^4_{i,j,\ell}\ |\  1 \leq i \leq m, 1 \leq j \leq 2, 1 \leq \ell \leq k\}.\\  
\end{eqnarray*}

For each $1 \leq i \leq 2(m-q) + 2n(q+1) + 4$ we add a voter with preferences \[\cdots \succ c \succ x^1_{i,1} \succ \cdots \succ x^1_{i,k-1} \succ c'\] and for each $1 \leq i \leq 2n(q+1) + 5$ we add a voter with preferences \[\cdots \succ c' \succ x^2_{i,1} \succ \cdots \succ x^2_{i,k-1} \succ c.\] For each $1 \leq i \leq n$ and $1 \leq j \leq 2(q+1)$, we add one voter with preferences \[\cdots \succ c' \succ x^3_{i,j,1} \succ \cdots \succ x^3_{i,j,k-1} \succ T_i.\] For each $1 \leq i \leq m$ and $1 \leq j \leq 2$, we add one voter with preferences \[\cdots \succ c \succ x^4_{i,j,1} \succ \cdots \succ x^4_{i,j,k-1} \succ b_i.\] We have thus, in essence, interchanged the role of approving $c$ with disapproving $c'$ and vice versa, as they have the same effect in relation to the candidates. We must now, as in the above case, add some voters to prevent our reduction from cheating by deleting buffer candidates.

For each $1 \leq i \leq 2(m-q) + 2n(q+1) + 4$, we add one voter with preferences \[\cdots \succ c' \succ x^1_{i,1} \succ \cdots x^1_{i,k}\] and for $1 \leq j \leq 2$, one voter with preferences \[\cdots \succ c' \succ x^4_{i,j,1} \succ \cdots \succ x^4_{i,j,k}.\]

Using the same argument as above, we see that hitting sets in $S$ correspond to valid destructive controls by either adding or deleting candidates.~\end{proof}

\subsection{Constructive Control}
We examine the reduction given in~\cite{BTT92} for constructive control by adding or deleting candidates in plurality elections and generalize it to cases of $k$-approval and $k$-veto elections.

In~\cite{BTT92}, the construction involved an election with established candidates $c$, $c'$, and $d$, and unestablished candidates corresponding to the elements of $S$. The idea was that by adding new candidates helps $c$ gain votes relative to $c'$, but also helps $d$ gain votes relative to $c$. Thus, candidate $d$ enforces our limit of adding voters, $q$.

In a constructive control by deleting candidates, there are a number of challenges to overcome in the addition of buffer voters and candidates. First, we must ensure that we cannot cheat by deleting either $c'$ or $d$. One way to overcome this is to clone both $c'$ and $d$, making copies of these candidates, all with the same function, such that we cannot delete all of the copies of either. Also, to extend to $k$-approval or $k$-veto by adding buffer candidates, we must also ensure that a valid control cannot involve deleting such buffer candidates. This can also be ensured by cloning each of the buffer candidates.

\begin{theorem}
$k$-approval elections are computationally resistant to constructive control by adding or deleting candidates.
\end{theorem}

\begin{proof}
Consider an instance of Hitting Set, where we are given a set $S = \{s_1,\ldots,s_m\}$, $n$ subsets of $S$, $T_1,\ldots,T_n$, and positive integer $1 \leq q \leq m$.

In our election, the candidate set will consist of

\begin{eqnarray*}
C = \{c\}&\cup&\{c'_1,\ldots,c'_{m-q+1}\} \\
         &\cup&\{d_1,\ldots,d_{m-q+1}\} \\
         &\cup& S \\
         &\cup& \{x_{i}\ |\  1 \leq i \leq k-1\} \\ 
\end{eqnarray*}

\noindent and the voter set will consist of

\begin{eqnarray*}
V&=&    \{(c \succ c'_1 \succ \cdots \succ c'_{k-1} \succ \cdots),(c'_{k} \succ \cdots \succ c'_{2k-1} \succ \cdots), \ldots,\\
&&(c'_{m-q-k+2} \succ \cdots \succ c'_{m-q+1} \succ \cdots)\ |\  1 \leq i \leq (n+2m)(m-q+1)-m+q \} \\
&& \\
 &\cup& \{(d_1 \succ \cdots \succ d_{k} \succ \cdots),(d_{k+1} \succ \cdots \succ d_{2k} \succ \cdots), \ldots,\\
&&(d_{m-q-k+1} \succ \cdots \succ c'_{m-q+1} \succ \cdots)\ |\  1 \leq i \leq (n+2m)(m-q+1) \} \\
&& \\
 &\cup& \{(T_i \succ x_1 \succ \cdots \succ x_{k-1} \succ c'_j \succ \cdots) \ |\  1 \leq i \leq n, 1 \leq j \leq m-q+1\} \\
&& \\
 &\cup& \{(s_i \succ x_1 \succ \cdots \succ x_{k-1} \succ c \succ \cdots) \ |\  1 \leq i \leq m\} \\
&&\\
 &\cup& \{(s_i \succ x_1 \succ \cdots \succ x_{k-1} \succ c'_j \succ \cdots) \ |\  1 \leq i \leq m, 1 \leq j \leq m-q+1\}. \\
\end{eqnarray*}

As seen in~\cite{BTT92}, adding a set of candidates corresponding to a hitting set of the established candidates $C - S$ makes $c$ a winner and vice versa, as this is the same election with additional buffer candidates and cloned candidates. In this case, each buffer candidate receives $(n+2m)(m-q+1)$ approvals, and cannot prevent $c$ from winning via any addition of candidates, as $c$ receives at least $(n+2m)(m-q+1)$ approvals.

We show that this is also a valid reduction showing that $k$-approval is computationally resistant to constructive control by deleting candidates. Consider a deletion of $m-q$ candidates from $C$. We see that at least one candidate $c'_i$ and $d_j$ must remain. Since the scores of each such candidate will remain the same after deletion, without loss of generality, we can assume that none of the candidates $c'_i$ or $d_j$ are deleted, since they either all win or all lose to $c$. Also without loss of generality, we can see that deleting any of the buffer candidates will not benefit $c$ in relation to $c'_i$ or $d_j$ more so than deleting an $s_j$ candidate, since it will remove one of the first $k$ candidates from the voters corresponding to each $T_i$. Thus, if there exists a deletion of $m-q$ candidates of $C$ making $c$ a winner, there also exists a deletion of $m-q$ candidates of $S$ making $c$ a winner. This corresponds to the complement of a hitting set of $S$.~\end{proof}

\begin{theorem}
$k$-veto elections are computationally resistant to constructive control by adding or deleting candidates.
\end{theorem}

\begin{proof}
We modify the construction above. The candidate and voter set in this case is defined as follows.

\begin{eqnarray*}
C = \{c\}&\cup&\{c'_1,\ldots,c'_{m-q+1}\} \\
         &\cup&\{d_1,\ldots,d_{m-q+1}\} \\
         &\cup& S \\
         &\cup& \{x_{i}\ |\  1 \leq i \leq k-1\} \\ 
\end{eqnarray*}

\begin{eqnarray*}
V&=&    \{(\cdots \succ c \succ c'_1 \succ \cdots \succ c'_{k-1}),(\cdots \succ c'_{k} \succ \cdots \succ c'_{2k-1}), \ldots,\\
&&(\cdots \succ c'_{m-q-k+2} \succ \cdots \succ c'_{m-q+1})\ |\  1 \leq i \leq (n+2m)(m-q+1) \} \\
&& \\
 &\cup& \{(\cdots \succ d_1 \succ \cdots \succ d_{k}),(\cdots \succ d_{k+1} \succ \cdots \succ d_{2k}), \ldots,\\
&&(\cdots \succ d_{m-q-k+1} \succ \cdots \succ c'_{m-q+1})\ |\  1 \leq i \leq (n+2m)(m-q+1)-m+q \} \\
&& \\
 &\cup& \{(\cdots \succ c \succ x_1 \succ \cdots \succ x_{k-1} \succ T_i) \ |\  1 \leq i \leq n\} \\
&& \\
 &\cup& \{(\cdots \succ c \succ x_1 \succ \cdots \succ x_{k-1} \succ s_i) \ |\  1 \leq i \leq m\} \\
&& \\
 &\cup& \{(\cdots \succ c'_j \succ x_1 \succ \cdots \succ x_{k-1} \succ s_i) \ |\  1 \leq i \leq m, 1 \leq j \leq m-q+1\} \\
\end{eqnarray*}~\end{proof}

\section{Weighted and Priced Cases of Election Misuse}\label{sect:weightprice}
We examine the problem of manipulation, bribery, and control for the case of weighted elections and voters with price tags for bribery. In weighted elections, each voters is given a weight and the points assigned to each voter is scaled by that weight. In some bribery problems, there is a cost associated with bribing each voter, and we would like to find the cheapest bribery. We would like to determine the effect on the complexity of manipulation.

In \cite{HH07}\ it is shown that a scoring protocol $\alpha = (\alpha_1,\ldots,\alpha_m)$ is computationally vulnerable to manipulation by weighted voters if and only if $\alpha_2=\ldots=\alpha_m$ (i.e., basically, only for elections with behavior similar to plurality and the trivial system in which each candidate receives the same score). Thus, $f(m)$-approval weighted elections can be computationally vulnerable to manipulation if and only if $f(m)=1$ for all $m \geq 1$. We conclude that $1$-approval (i.e., plurality) is the only approval-based scoring protocol that is computationally vulnerable to manipulation.

In bribery problems, voters can have both weights and prices. In this case, we pay a price for bribing each voter, and we have a budget $q$. We may bribe as many votes as we want so long as this budget is not exceeded.

In~\cite{FHH06}\, it is seen that weighted scoring protocols elections are computationally vulnerable to bribery iff $\alpha_2=\ldots=\alpha_m$, and vulnerable to \$bribery iff $\alpha_1=\ldots=\alpha_m$. Thus, $1$-approval is also the only weighted approval-based scoring protocol that is computationally vulnerable to bribery. No nontrivial weighted approval-based scoring protocol are computationally vulnerable to \$bribery.
\begin{theorem}
$1$-approval and $1$-veto (i.e., plurality and veto) elections are computationally vulnerable to constructive \$bribery.
\end{theorem}

\begin{proof}
The case of \$bribery for $1$-approval is shown in~\cite{FHH06}. In $1$-veto, as in the unweighted case, we must bribe voters vetoing $p$ and give the vetoes to the candidate with the fewest vetoes. Since the voters are indistinguishable by weight, we simply bribe the cheapest voters available.~\end{proof}

\begin{theorem}
Simple Weighted $b$-Edge Cover of Multigraphs is polynomial-time many-one reducible to \$bribery in $2$-veto elections.\end{theorem}
\begin{proof}
Consider an instance of Simple Weighted $b$-Edge Cover of Multigraphs: Let $G = (V,E)$ be a multigraph, and let $b(v)$ denote the multiplicity that $v \in V$ is to be covered, and $w(e)$ denote the weight of edge $e \in E$. We wish to find a $b$-edge cover of $G$ with edges of total weight at most $q$.

We construct a $2$-veto election as follows. The voter set will be given by $V \cup \{p,p',b,b'\}$. We will be bribing to elect $p$. Note that the voter set also includes each vertex of $G$ and three buffer candidates, $p'$, $b$, and $b'$.

Initially, there are $||E||$ voters of price $q+1$ who veto $\{p,p'\}$ and $||E||$ voters of price $q+1$ who veto $\{b,b'\}$. For each edge $e = \{v_1,v_2\} \in E$ of weight $w(e)$, let there be one voter of price $w(e)$ who vetoes $\{v_1,v_2\}$. Finally, for each vertex $v \in V$, let there be $||E|| + b(v) - \mathtt{deg}(v)$ voters of price $q+1$ who veto $\{v,b\}$. $\mathtt{deg}(v)$ is the degree of vertex $v$ in $G$.

Initially, $p$ has $||E||$ vetoes, while each of $p'$, $b$, and $b'$ have at least $||E||$ vetoes. Each $v \in V$ has exactly $||E|| + b(v)$ vetoes.

Consider a weighted $b$-edge cover of $G$ of weight at most $q$. It is possible to bribe the voters corresponding to this $b$-edge cover, at cost at most $q$, to veto $\{b,b'\}$ instead. On the other hand, consider a bribery of cost at most $q$ which makes $p$ a winner. Clearly, only the voters corresponding to the edges in $E$ can be bribed, as all of the other voters have price $q+1$. It must be the case that for each $v \in V$, at least $b(v)$ voters are bribed. This corresponds to a $b$-edge cover.

This shows that \$bribery of $2$-veto elections is at least as hard as Simple Weight $b$-Edge Cover of Multigraphs. However, the complexity of this problem remains an open problem.~\end{proof}

In the next result, we find two problems of control that are equivalent to another variation of $b$-Edge Cover we defined as Simple $b$-Edge Weighted Cover of Multigraphs. In this problem, each edge counts differently toward the multiplicity of the vertex being covered. Edges are, however, unweighted, in that a typical instance will seek a covering by some number, $q$, of edges.

\begin{theorem}
Control by adding weighted voters in $2$-veto elections, and by deleting weighted voters in $2$-approval elections, is polynomial-time equivalent to Simple $b$-Edge Weighted Cover of Multigraphs.
\end{theorem}

\begin{proof}
Consider the case of weighted $2$-veto control by adding voters.

Consider an election $E=(C,V)$. For each candidate $c \in C$, define $\mathrm{Score}(c)$ to be the current score of candidate $c$ in $E$. Without loss of generality, we add only voters that do not veto $p$. To ensure $p$'s victory, for each candidate $c \neq p$ currently beating $p$, we must add vetoes of weight totaling at least $\mathrm{Score}(c) - \mathrm{Score}(p)$.

We construct a graph of vertex set $C - \{p\}$. For each voter $v$ of weight $w(v)$ vetoing candidates $\{a, b\}$, we add an edge of multiplicity $w(v)$ connecting vertices $a$ and $b$. For each candidate $c \in C - \{p\}$, we assign the $b$-value of its corresponding vertex $b(c) = \mathrm{max}(\mathrm{Score}(c) - \mathrm{Score}(p),0)$.

Observe that in a valid covering of at most $q$ vertices, adding these $q$ voters ensures that each candidate $c \neq p$ receives vetoes of total weight at least $\mathrm{max}(\mathrm{Score}(c) - \mathrm{Score}(p),0)$, and thus $p$ beats $c$. Conversely, in any addition of voters, without loss of generality, we may disregard any additional voters whom veto $p$. Since each candidate $c \neq p$ must receive vetoes of total weight at least $\mathrm{max}(\mathrm{Score}(c) - \mathrm{Score}(p),0)$ to ensure $p$'s victory, edges corresponding to the additional voters must correspond to a $b$-Edge cover.

Conversely, consider an instance of Simple $b$-Edge Weighted Cover of Multigraphs. Let $G = (V,E)$ be a multigraph. For each $v \in V$, let $b(v)$ denote the minimum multiplicity of edges that $v$ is to be covered, and for each edge $e \in E$, let $w(e)$ denote the multiplicity of the edge $e$.

Let $M = \mathrm{max}_{v}{b(v)}$ denote the largest $b$-value of the vertices. Construct an election over the candidates $\{p\} \cup V$ such that initially $p$ receives $M$ vetoes, and each $v \in V$ receives $M - b(v)$ vetoes. For each $e =(v_1,v_2)\in E$, let there be one unestablished voter of weight $w(e)$ vetoing $\{v_1,v_2\}$.

Consider an addition of at most $q$ voters that elects $p$. As $p$ cannot gain any vetoes from the addition of the voters, clearly, each $v \in V$ must gain at least $b(v)$ vetoes. This corresponds to a $b$-Edge Cover of at most $q$ edges. Conversely, adding at most $q$ edges corresponding to a $b$-Edge Cover adds at least $b(v)$ vetoes to each $v \in V$, giving each $v \in V$ at least $M$ vetoes, and electing $p$.

A similar construction can be used for weighted $2$-approval control by deleting voters.~\end{proof}

Unfortunately, unlike the unweighted cases, this reduction cannot be easily modified for the case of adding weighted voters of $3$-veto elections or deleting weighted voters in $3$-approval elections. This is because it is not possible to compute the final score of $p$ following the addition or deletion of voters in these cases. Recall that in unweighted cases, it can be assumed that without loss of generality we will alter as many voters as possible.

\begin{theorem}
$k$-approval elections for $k \geq 3$ and $k$-veto for $k \geq 4$ are computationally resistant to \$bribery.
\end{theorem}
\begin{proof}
This problem is a generalization of unpriced bribery for the same election systems, which is NP-hard.
\end{proof}

\begin{theorem}
Weighted $1$-approval and $1$-veto elections are computationally vulnerable to constructive control by adding and deleting voters.
\end{theorem}
\begin{proof}
We consider the four cases separately. In $1$-approval constructive control by adding voters, it is only relevant to add voters approving our distinguished candidate $p$. We thus add the heaviest voters approving $p$, up to our quota or until we are out of such voters. Control is possible iff $p$ wins in this case.

In $1$-veto constructive control by adding voters, for each candidate $p'$ beating $p$, we must eventually add votes vetoing $p'$. Without loss of generality we can add the heaviest vote. Our algorithm thus works as follows. Until $p$ is winning, for each candidate $p'$ currently beating $p$, we add the vote vetoing $p'$ with the heaviest weight.

In $1$-approval constructive control by deleting voters, we see that for each candidate $p'$ beating $p$, we must eventually delete a vote approving $p'$. Without loss of generality we can delete the heaviest vote. Our algorithm thus works as follows. Until $p$ is winning, for each candidate $p'$ beating $p$, we delete the vote approving $p'$ with the heaviest weight.

In $1$-veto constructive control by deleting voters, we delete the heaviest votes vetoing $p$.~\end{proof}

\begin{theorem}
Weighted $k$-approval elections for $k \geq 4$ and $k$-veto elections for $k \geq 3$ are computationally resistant to control by adding voters. Similarly, weighted $k$-approval elections for $k \geq 3$ and $k$-veto elections for $k \geq 4$ are computationally resistant to control by deleting voters.
\end{theorem}
\begin{proof}
These problems are generalizations of the corresponding unweighted problems for the same election systems, which are NP-hard.
\end{proof}

\section{\uppercase{Results and Discussion}}

These results show the variance of complexity of misuse among different problems in election systems of the form $k$-approval, $k$-veto, and $f(m)$-approval, and give a result of complexity for infinitely many scoring protocols of an unbounded number of candidates. There are a few interesting cases: These manipulations can either be easy by a simple greedy algorithm, equivalent to a corresponding variation of $b$-Edge Cover, which is easy for the unweighted and unpriced cases but unknown for the weighted or priced variations, or hard by reduction from Set Cover. Manipulations involving the candidate sets are always difficult as a result of reductions from Hitting Set. These results show a connection between election manipulation and seemingly unrelated graph theory problems.

Our results demonstrate the strengths and weaknesses of approval-based election systems, and we hope they can lead to further generalizations and possible developments of systems that better resist such attacks. There are, however, three open problems: that of control by adding voters in $3$-approval elections, control by deleting voters in $3$-veto elections, and \$bribery of $3$-veto elections. These cases cannot be evaluated using the Edge Cover variants, because, for instance, to bribe voters in $3$-veto elections, we must choose between those that distinguish three of the candidates other than $p$. This cannot be accomplished with Simple Weighted $b$-Edge Cover of Multigraphs. The problem of X3C also does not appear to reduce to this or the other cases. We leave these as open problems. The complexity of Simple Weighted $b$-Edge Cover of Multigraphs and Simple $b$-Edge Weighted Cover of Multigraphs is also left as an open problem. We believe it both problems are likely to be polynomial-time computable, due to the results for other variations of $b$-Edge Cover and the connections between this problem and linear programming.

It is important to realize that NP-completeness only addresses the worst-case complexity of a given problem, and does not take into consideration the distribution of problems that might be given. Some simple distributions were considered in~\cite{Walsh,FKN}, and it may be of interest to characterize the complexity of more interesting and realistic distributions, depending upon the application.

This model also makes the assumption that in a $k$-approval election, each voter may vote for any combination of the $k$ candidates independently. We know that in practice, most elections do not follow this principle, and voters in most realistic elections have highly correlated preferences (e.g., nearly single-peaked preferences~(see~\cite{BH06})). In~\cite{FHHR}, several otherwise hard problems in election manipulation are shown to be easy when restricted to the case of single-peaked preferences. It is shown that constructive or destructive control by adding or deleting candidates is easy for plurality elections, as well as weighted manipulation in some scoring protocols other than plurality, become easy when the voters (including the manipulators) are restricted to having single-peaked preferences. It may thus be of interest to characterize the complexity properties of manipulation in a more realistic distribution of voter preferences.

\section{\uppercase{Acknowledgements}}
\noindent We wish to offer our special thanks to Dr. E. Hemaspaandra and anonymous referees for their helpful comments. We especially thank Dr. E. Hemaspaandra for pointing out the link to $b$-Edge Cover.

\end{document}